\documentclass[letterpaper, 10 pt, journal, twoside]{IEEEtran} 
\usepackage{times}

\usepackage[T1]{fontenc}
\usepackage[english]{babel}
\usepackage[pdftex]{graphicx}   
\usepackage{epsfig}             
\usepackage{amsmath}            
\usepackage{amssymb}            
\usepackage{mathtools}
\usepackage[numbers]{natbib}
\usepackage{multicol}
\usepackage[bookmarks=true]{hyperref}
\usepackage[dvipsnames]{xcolor}
\usepackage{siunitx}
\usepackage{lipsum}
\usepackage{empheq}
\usepackage[customcolors,markings]{hf-tikz}
\usepackage{nicematrix}
\usepackage{nicefrac, xfrac}
\usepackage[commandnameprefix=ifneeded, final]{changes}

\usetikzlibrary{calc,backgrounds}

\usepackage{bm}

\usepackage[T1]{fontenc}

\usepackage{MnSymbol}
\usepackage{mathdots}

\usepackage{amsthm}

\providecommand{\R}{\mathbb{R}}

\providecommand{\SO}{\mathbf{SO}}

\providecommand{\SE}{\mathbf{SE}}

\providecommand{\IN}{\mathbf{IN}}
\providecommand{\grpG}{\mathbf{G}}

\providecommand{\grpSD}{\mathbf{G_{SD}}}

\providecommand{\gothse}{\mathfrak{se}}

\providecommand{\gothg}{\mathfrak{g}}

\providecommand{\so}{\mathfrak{so}}
\providecommand{\se}{\mathfrak{se}}

\providecommand{\sdpgrpG}{\mathbf{G}_{\mathfrak{g}}^{\ltimes}}

\providecommand{\calG}{\mathcal{G}}

\providecommand{\calM}{\mathcal{M}}

\providecommand{\torSE}{\mathcal{SE}}

\providecommand{\torIN}{\mathcal{IN}}

\providecommand{\vecL}{\mathbb{L}}

\providecommand{\td}{\mathrm{d}}

\providecommand{\Fr}[2]{\left. \mathrm{D}_{#1} \right|_{#2}}

\usepackage{accents}
\makeatletter
\providecommand{\scirc}{%
    \hbox{\fontfamily{\rmdefault}\fontsize{0.4\dimexpr(\f@size pt)}{0}\selectfont{\raisebox{-0.52ex}[0ex][-0.52ex]{$\circ$}}}}

\makeatother

\mathchardef\mhyphen="2D

\theoremstyle{plain}
\newtheorem{theorem}{Theorem}[section]

\newtheorem{lemma}[theorem]{Lemma}

\theoremstyle{definition}

\theoremstyle{remark}

\newcommand{\set}[2]{\left\{ #1 \;\middle|\; #2 \right\}}

\newcommand{\st}{\;|\;}

\newcommand{\AtoB}[2]{\;:\;#1\;\rightarrow\;#2}

\newcommand{\chrulefill}{%
  \leaders\hrule height \dimexpr\fontdimen22\scriptscriptfont2+0.2pt\relax
                 depth -\dimexpr\fontdimen22\scriptscriptfont2-0.2pt\relax
          \hfill}
\newcommand{\xoverline}[2]{%
  \mathord{
    \vbox{\offinterlineskip
      \halign{##\cr
        \chrulefill$\,\scriptscriptstyle#1\,$\chrulefill\cr
        \noalign{\kern.2ex}
        $#2$\cr
      }%
    }%
  }%
}

\providecommand{\td}{\mathrm{d}}

\providecommand{\Fr}[2]{\mathrm{D}_{#1}\big|_{#2}}

\newcommand{\eye}{\mathbf{I}}

\newcommand{\Adsym}[2]{\mathrm{Ad}_{#1}\left[#2\right]}

\newcommand{\AdMsym}[1]{\mathbf{Ad}^\vee_{#1}}
\newcommand{\adMsym}[1]{\mathbf{ad}^\vee_{#1}}

\providecommand{\xizero}{\mathring{\xi}}

\newcommand{\Vector}[3]{\prescript{#1}{}{\bm{#2}}_{#3}}
\newcommand{\dotVector}[3]{\prescript{#1}{}{\dot{\bm{#2}}}_{#3}}
\newcommand{\hatVector}[3]{\prescript{#1}{}{\hat{\bm{#2}}}_{#3}}

\newcommand{\Pose}[2]{\prescript{#1}{}{\mathbf{T}}_{#2}}
\newcommand{\PoseS}[2]{\prescript{#1}{}{\mathbf{S}}_{#2}}
\newcommand{\PoseP}[2]{\prescript{#1}{}{\mathbf{P}}_{#2}}
\newcommand{\VAtt}[2]{\prescript{#1}{}{\mathbf{V}}_{#2}}
\newcommand{\Rot}[2]{\prescript{#1}{}{\mathbf{R}}_{#2}}
\newcommand{\dotPose}[2]{\prescript{#1}{}{\dot{\mathbf{T}}_{#2}}}

\newcommand{\hatPoseP}[2]{\prescript{#1}{}{\hat{\mathbf{P}}_{#2}}}

\newcommand{\frameofref}[1]{\{$#1$\}}

\providecommand{\bw}{\Vector{}{b}{\bm{\omega}}}
\providecommand{\ba}{\Vector{}{b}{a}}
\providecommand{\angv}{\Vector{}{\omega}{}}
\providecommand{\acc}{\Vector{}{a}{}}
\providecommand{\ethree}{\Vector{}{e}{3}}
\newcommand{\figref}[1]{Fig.~\ref{fig:#1}}
\newcommand{\secref}[1]{Sec.~\ref{sec:#1}}
\newcommand{\equref}[1]{Equ.~(\ref{eq:#1})}

\newcommand{\tabref}[1]{Tab.~\ref{tab:#1}}

\usepackage{acro}
\input{acro/acro.sty}

\usepackage[caption=false,font=footnotesize]{subfig}

\usepackage{array}
\usepackage{ctable}
\newcolumntype{P}[1]{>{\centering\arraybackslash}p{#1}}
\newcolumntype{M}[1]{>{\centering\arraybackslash}m{#1}}

\DeclareSIUnit{\sqrts}{\ensuremath{\sqrt{\text{\second}}}}
\newcommand{\changed}[2]{\deleted{#1} \added{#2}}

\let\oldcite\cite
\renewcommand{\cite}[1]{\mbox{\oldcite{#1}}}

\newcommand{\stkout}[1]{\ifmmode\text{\sout{\ensuremath{#1}}}\else\sout{#1}\fi}
\setdeletedmarkup{\stkout{#1}}

\hypersetup{
  pdfinfo={
    Title={},
    Author={Alessandro Fornasier, Pieter van Goor, Eren Allak, Robert Mahony and Stephan Weiss},
    Subject={Equivariant Filter for Visual Inertial Odometry},
    Keywords={Equivariant Filter, Equivariant System, State Estimation, Visual Inertial Odometry, VIO, Simultaneous Localization and Mapping, SLAM}
  }
}

\title{MSCEqF: A Multi State Constraint Equivariant Filter for Vision-aided Inertial Navigation}

\author{Alessandro Fornasier$^{1}$, Pieter van Goor$^{2}$, Eren Allak$^{1}$, Robert Mahony$^{2}$ and Stephan Weiss$^{1}$%
\thanks{Manuscript received: July, 2, 2023; Revised October, 13, 2023; Accepted November, 13, 2023.}%
\thanks{This paper was recommended for publication by Editor Pascal Vasseur upon evaluation of the Associate Editor and Reviewers' comments.
This work was supported by the European Union’s Horizon 2020 research and innovation program under grant agreement 871260 (BugWright2), and by the Army Research Office under Cooperative Agreement Number W911NF-21-2-0245. The views and conclusions contained in this document are those of the authors and should not be interpreted as representing the official policies, either expressed or implied, of the Army Research Office or the U.S. Government. The U.S. Government is authorized to reproduce and distribute preprints for Government purposes notwithstanding any copyright notation herein.
}%
\thanks{$^{1}$Alessandro Fornasier, Eren Allak and Stephan Weiss are with the Control of Networked Systems Group, University of Klagenfurt, Austria. {\tt\small \{name.surname\}@ieee.org}}%
\thanks{$^{2}$Pieter van Goor and Robert Mahony are with the System Theory and Robotics Lab, Australian National University, Australia. {\tt\small \{name.surname\}@anu.edu.au}}%
}%

\begin{document}
\bstctlcite{BSTcontrol}

\maketitle



\begin{abstract}
This letter re-visits the problem of \ac{vins} and presents a novel filter design we dub the multi state constraint \emph{equivariant} filter (MSC\emph{EqF}, in analogy to the well known \acs{msckf}).
We define a symmetry group and corresponding group action that allow specifically the design of an \acl{eqf} for the problem of \ac{vio} including \acs{imu} bias, and camera intrinsic and extrinsic calibration states. In contrast to state-of-the-art \ac{iekf} approaches that simply tack \ac{imu} bias and other states onto the $\SE_2(3)$ group, 
our filter builds upon a symmetry that properly includes all the states in the group structure. 
Thus, we achieve improved behavior, particularly when linearization points largely deviate from the truth (i.e., on transients upon state disturbances).
Our approach is \emph{inherently consistent} even during convergence phases from significant errors without the need for error uncertainty adaptation, \acl{oc}, or other consistency enforcing techniques. This leads to greatly improved estimator behavior for significant error and unexpected state changes during, e.g., long-duration missions. We evaluate our approach with a multitude of different experiments using three different prominent real-world datasets.

\end{abstract}

\vspace{-10pt}

\section{Introduction and Related Work}

\IEEEPARstart{I}{n} the past years, \ac{vins} have shown remarkable success in estimating the position and orientation of robots by relying only on low-cost and lightweight \acp{imu} and cameras.

Popular algorithms for \ac{vins} include \acf{vio} and \ac{vislam}. 
\ac{vio} focuses only on the local surroundings and is, therefore, computationally simpler, less accurate, and it suffers from accumulated drift.
\ac{vins} algorithms can also suffer from inconsistencies~\cite{Hesch2014ConsistencyNavigation}. The classical \ac{ekf}-\acs{slam} algorithm suffers from overconfidence due to spurious information gain along the unobservable directions~\cite{Huang2008AnalysisSLAM}.
Different solutions have been proposed in literature to overcome the problems caused by inconsistencies. By manipulating the linearization point and enforcing the correct number of unobservable directions for the linearized system, Huang \emph{et al.} introduced the \ac{fej}~\cite{huang2009first}, whereas Hesch \emph{et al.} the \acf{oc}~\cite{Hesch2014ConsistencyNavigation} as techniques aiming at solving the inconsistency issue at the cost of sub-optimal linearization points. 
More recently, in~\cite{7523335}, Barrau and Bonnabel introduced the \ac{iekf} and showed that exploiting the natural symmetry of group affine systems leads to algorithms that are inherently consistent~\cite{barrau2015ekf}. Although the \ac{iekf} theory does not apply to \ac{ins} when \ac{imu} bias are explicitly considered, many authors~\changed{\cite{doi:10.1177/0278364919894385, Wu2017AnConsistency, Liu2023InGVIO:Odometry, Yang2022DecoupledNavigation}}{\cite{Heo2018ConsistentGroup, Brossard2018InvariantSLAM, Brossard2018UnscentedOdometry, doi:10.1177/0278364919894385, Wu2017AnConsistency, Liu2023InGVIO:Odometry, Yang2022DecoupledNavigation}} have exploited the Imperfect-\ac{iekf} framework~\cite{barrau:tel-01247723} to design \ac{vins} algorithms.

In very recent research, van Goor \emph{et al.} introduced the \ac{eqf}~\cite{VanGoor2020EquivariantSpaces, vanGoor2022EquivariantEqF} as a general filter design for systems on homogeneous spaces, and proposed a symmetry for fixed landmark measurements in the context of \ac{vislam}~\cite{9029435, van2020observer, van2020constructive, vanGoor2021AnOdometry, vanGoor2023EqVIO:Odometry}. Later, Fornasier \emph{et al.} proposed a novel symmetry for \ac{ins} that couples navigation states and \ac{imu} bias and developed an \ac{eqf} design for \ac{ins}~\cite{Fornasier2022EquivariantBiases, Fornasier2022OvercomingCalibration} that proved superior to state-of-the-art in terms of robustness to wrong initialization, transient behavior, and consistency properties. \added{In a very recent research study~\cite{Fornasier2023EquivariantSystems}, the same authors analyzed the theoretical properties of different symmetry groups when employed in designing filters for inertial navigation systems, and provided a discussion of the relative strengths and weaknesses of different filter algorithms.}

For \emph{vision} aided \ac{ins} systems, however, the lack of robustness against unexpected disturbances and the requirement for sophisticated tuning for a given environment and setup remain important limitations.
Real-world deployments are typically constrained to precise tuning and highly engineered codebases, where the core \ac{vio} algorithm is encompassed by numerous modules responsible for tasks such as initialization, failure detection, algorithm reset, and more. 
A \emph{people's \acl{vio}}, that is, an algorithm whose operation requires minimal knowledge, little to no tuning, and yet still functions in many different real-world scenarios, would enable a whole new tranch of real-world applications without the requirement of having highly trained engineers available. 
The present letter \added{builds upon the recent results in~\cite{Fornasier2022EquivariantBiases, Fornasier2022OvercomingCalibration, Fornasier2023EquivariantSystems}} and is a step towards enabling this goal. 

This perspective shifts the evaluation of algorithm performance from measures such as \ac{rmse}, accuracy, and precision, to measures such as the likelihood of failure for poor initial conditions or poor calibration. We acknowledge that state-of-the-art \ac{vins} approaches reached a plateau in the former metrics, but there is still a large room for improvement in the latter metrics. Furthermore, this letter does not claim completeness in comparative evaluations, rather, we present here our novel findings enabling a \ac{msceqf} as a step towards the \emph{people's \ac{vio}}; compare it against OpenVINS~\cite{Geneva2020OpenVINS:Estimation}, the best open-source available \ac{msckf}~\cite{Mourikis2007ANavigation}, and see an extensive comparison covering all suitable approaches as a work that goes beyond the scope of this letter.

Apart from the different metric evaluation, this work differentiates itself from state-of-the-art by extending insights on symmetries and \ac{eqf} design for fixed landmark \ac{vins}~\cite{vanGoor2021AnOdometry, vanGoor2023EqVIO:Odometry} and \ac{ins} including \ac{imu} bias into the symmetry~\changed{\cite{Fornasier2022EquivariantBiases, Fornasier2022OvercomingCalibration}}{\cite{Fornasier2022EquivariantBiases, Fornasier2022OvercomingCalibration, Fornasier2023EquivariantSystems}} to the idea of a multi state constraint but \emph{equivariant} \ac{vins}. 
%
To the best of our knowledge, the resulting algorithm is the first \added{ever, equivariant} multi state constraint
\changed{\acl{eqf}}{filter} for \ac{vio}. Our approach, dubbed \ac{msceqf}, \added{leverages a semi-direct product symmetry group, yielding improved linearized error dynamics when compared to other filter types~\cite{Fornasier2023EquivariantSystems}. Hence, the \ac{msceqf}} demonstrates consistency naturally without artificial changes of linearization points and very high robustness to poor extrinsic calibration. It not only handles significant absolute (calibration) errors but also addresses the concept of \emph{dealing with ``you don't know what you don't know''}, such as errors exceeding the prior covariance (e.g., sudden changes of calibrations states due to a disturbance during the operational phase of the robotic platform, where the state has converged already and the covariance has shrunk). 


To summarize, with this work, we make the following contributions:

{\bf (i):} We introduce the \ac{msceqf}; a novel multi state constraint \acl{vins} based on the \acl{eqf} framework, with camera and IMU self-calibration capabilities.
\deleted{Contrary to prior work, we derived the filter matrices in analytical form without resorting to numerical differentiation and ensuring no simplifications were made. The analytical derivation ensures higher portability and lower computational complexity enabling the use of \ac{msceqf} on compute-limited hardware.}

{\bf (ii):} We demonstrate that the proposed \ac{msceqf} achieves state-of-the-art accuracy, with superior robustness to significant absolute errors, as well as errors exceeding the prior covariance. 

Our experiments show that the \ac{msceqf} can be directly deployed in real-world scenarios with little tuning and no additional health-check modules. 
Furthermore, we show that the proposed \ac{msceqf} is a naturally consistent filter without the need for \ac{fej}, \ac{oc}, or other heuristic techniques.
\deleted{Furthermore} We implemented our framework as a stand-alone C++ library, and we made it source-available to the community\footnotemark. 
Wrappers for the standard middle-ware (e.g., ROS1, ROS2, etc.) will be provided such that code is available for direct use and comparison against other approaches. 
\added{We derived the filter matrices in analytical form without resorting to numerical differentiation, leading to code with higher portability and lower computational complexity, appropriate for compute-limited hardware, such as nano-drones, augmented reality devices, etc.}
\deleted{With this work, we focus on a statistical filter design rather than on a factor graphs algorithm as a relevant choice for applications where the computing power is still a limiting factor, such as nano-drones, augmented reality devices, etc. }

\footnotetext{\href{https://github.com/aau-cns/MSCEqF}{https://github.com/aau-cns/MSCEqF}}
\section{Mathematical Preliminaries and Notation}

\subsection{Vector and matrix notation}
Vectors describing physical quantities expressed in frame of reference \frameofref{A} are denoted by $\Vector{A}{v}{}$.
Rotation matrices encoding the orientation of a frame of reference \frameofref{B} with respect to a reference \frameofref{A} are denoted by $\Rot{A}{B}$; in particular, \deleted{$\Rot{A}{B}$ expresses a vector $\Vector{B}{v}{}$ defined in the \frameofref{B} frame of reference into a vector} ${\Vector{A}{v}{} = \Rot{A}{B}\Vector{B}{v}{}}$ \deleted{expressed in the \frameofref{A} frame of reference}.
${\eye_n \in \R^{n \times n}}$ denotes the $n$-dim identity matrix, and ${\mathbf{0}_{n \times m} \in \R^{n \times m}}$ denotes the zero matrix with ${n}$ rows and ${m}$ columns.

\subsection{Lie theory}
A Lie group $\grpG$ is a smooth manifold endowed with a smooth group structure.
For any $X, Y \in \grpG$, the group multiplication is denoted $XY$, the group inverse $X^{-1}$ and the identity element $I$. 

Given a Lie group $\grpG$, $\calG$ denotes the $\grpG$-Torsor~\cite{Mahony2013ObserversSymmetry}\deleted{, which is defined as the set of elements of $\grpG$ with the smooth manifold structure but without the group structure}.

\deleted{In this letter we will limit our consideration to matrix Lie groups and products of matrix Lie groups.}

For a given Lie group $\grpG$, the Lie algebra $\gothg$ is a vector space corresponding to the tangent space at the identity of the group, together with a bilinear non-associative map ${[\cdot, \cdot] \AtoB{\gothg\times\gothg}{\gothg}}$ called the \emph{Lie bracket}.
The Lie algebra $\gothg$ is isomorphic to a vector space $\R^{n}$ of dimension ${n =\mathrm{dim}\left(\gothg\right)}$.

Define the \emph{wedge} map and its inverse, the \emph{vee} map as linear isomorphisms between the vector space and the Lie algebra 
\begin{equation*}
    \left(\cdot\right)^{\wedge} \AtoB{\R^{n}}{\gothg} ,\qquad \left(\cdot\right)^{\vee} \AtoB{\gothg}{\R^{n}},
\end{equation*} 
such that $(\Vector{}{u}{}^\wedge)^\vee = \Vector{}{u}{}$, for all $\Vector{}{u}{} \in \mathbb{R}^n$.

For any $X, Y \in \grpG$, define the left and right translations
\begin{align*}
    &\textrm{L}_{X} \AtoB{\grpG}{\grpG}, \qquad \textrm{L}_{X}\left(Y\right) = XY ,\\
    &\textrm{R}_{X} \AtoB{\grpG}{\grpG}, \qquad \textrm{R}_{X}\left(Y\right) = YX . 
\end{align*}

The Lie group (`big') Adjoint matrix is defined by
\begin{equation*}
    \AdMsym{X} \AtoB{\R^{n}}{\R^{n}}, \qquad \AdMsym{X}\Vector{}{u}{} = \left(\td \textrm{L}_{X} \td \textrm{R}_{X^{-1}}\left[\Vector{}{u}{}^{\wedge}\right]\right)^\vee ,
\end{equation*}
for every $X \in \grpG$ and ${\Vector{}{u}{}^{\wedge} \in \gothg}$, where $\td \textrm{L}_{X}$, and $\td \textrm{R}_{X}$ denote the differentials of the left, and right translation, respectively.

The Lie algebra (`little') adjoint matrix is defined by
\begin{equation*}
    \adMsym{\Vector{}{u}{}} \AtoB{\R^{n}}{\R^{n}},\qquad \adMsym{\Vector{}{u}{}}\bm{v} = \left[\Vector{}{u}{}^{\wedge}, \bm{v}^{\wedge}\right]^{\vee},
\end{equation*}
for every ${\Vector{}{u}{}, {\Vector{}{v}{} \in \R^n}}$.
\deleted{The two adjoint matrices commute in the sense that ${\AdMsym{X}\adMsym{\Vector{}{u}{}}\Vector{}{v}{} = \adMsym{\AdMsym{X}\Vector{}{u}{}}\AdMsym{X}\Vector{}{v}{}}$,
for all $\Vector{}{u}{}, \Vector{}{v}{} \in \R^n$ and $X \in \grpG$.}

\subsection{Important matrix Lie groups}
The special orthogonal group $\SO(3)$, special Euclidean group $\SE(3)$, extended special Euclidean group $\SE_2(3)$, and their respective Lie algebras are defined, in matrix form, by
\begin{align*}
    \SO(3) &= \set{\mathbf{A} \in \R^{3\times3}}{
    \mathbf{A} \mathbf{A}^\top = \mathbf{I}_3, \; \det(\mathbf{A} ) = 1
    }, \\
    \so(3) &= \set{\bm{\omega}^{\wedge} \in \R^{3\times3}}{
    \bm{\omega}^{\wedge} = - {\bm{\omega}^{\wedge}}^\top
    }, \\
    \SE(3) &= \set{
    \begin{bmatrix}
        \mathbf{A} & \Vector{}{a}{} \\ \mathbf{0}_{1\times 3} & 1
    \end{bmatrix}
    \in \R^{4\times4}}{
    \mathbf{A} \in \SO(3), \; \Vector{}{a}{} \in \R^3
    }, \\
    \se(3) &= \set{ 
    \begin{bmatrix}
        \bm{\omega}^{\wedge} & \Vector{}{v}{} \\ \mathbf{0}_{1\times 3} & 0
    \end{bmatrix}
    \in \R^{4\times4}}{
    \bm{\omega}^{\wedge} \in \so(3), \; \Vector{}{v}{} \in \R^3
    }, \\
    \SE_2(3) &= \set{
    \begin{bmatrix}
        \mathbf{A} & \begin{matrix}\Vector{}{a}{} & \Vector{}{b}{}\end{matrix} \\ \mathbf{0}_{2\times 3} & \eye_2 
    \end{bmatrix}
    \in \R^{5\times5}}{
    \mathbf{A} \in \SO(3), \; \Vector{}{a}{}, \Vector{}{b}{} \in \R^{3}
    }, \\
    \se_2(3) &= \set{ 
    \begin{bmatrix}
        \bm{\omega}^{\wedge} & \begin{matrix}\Vector{}{v}{} & \Vector{}{w}{}\end{matrix} \\ \mathbf{0}_{2\times 3} & \mathbf{0}_{2\times 2}
    \end{bmatrix}
    \in \R^{5\times5}}{
    \bm{\omega}^{\wedge} \in \so(3), \; \Vector{}{v}{}, \Vector{}{w}{} \in \R^{3}
    }.
\end{align*}


\deleted{A right group action of a Lie group $\grpG$ on a differentiable manifold $\calM$ is a smooth map ${\phi \AtoB{\grpG\times \calM}{\calM}}$ that satisfies}
\begin{equation*}
    \deleted{\phi\left(I, \xi\right) = \xi ,\qquad \phi\left(X, \phi\left(Y, \xi\right)\right) = \phi\left(YX, \xi\right) ,}
\end{equation*}
\deleted{for all $X,Y \in \grpG$ and $\xi \in \calM$.
A right group action $\phi$ induces a family of diffeomorphisms ${\phi_X \AtoB{\calM}{\calM}}$ and smooth projections ${\phi_{\xi} \AtoB{\grpG}{\calM}}$.
The group action $\phi$ is said to be transitive if the induced projection $\phi_{\xi}$ is surjective.
In this case, $\calM$ is a homogeneous space of $\grpG$.}

\subsection{Semi-direct Bias group $\grpSD \coloneqq \SE_2(3) \ltimes \gothse(3)$}
The \emph{Semi-direct Bias group} ${\grpSD \coloneqq \SE_2(3) \ltimes \gothse(3)}$ \added{introduced in~\cite{Fornasier2023EquivariantSystems}}, is a group structure on the tangent bundle\deleted{~\cite{Jayaraman2020Black-ScholesGroup, ng2019attitude, Ng2020EquivariantGroups, Ng2020PoseKinematics}} ${\sdpgrpG := \mathbf{G} \ltimes \gothg}$ given by the semi-direct product of a group $\grpG$ with \changed{its Lie algebra $\gothg$}{a Lie subalgebra $\gothg$}. 

\changed{In particular, let $A,B \in \grpG$ and $a,b \in \gothg$ and define $X = \left(A, a\right)$ and $Y = \left(B, b\right)$ elements of the symmetry group $\sdpgrpG$. 
Group multiplication is defined to be the semi-direct product ${XY = \left(AB, a + \Adsym{A}{b}\right)}$. 
The inverse element is ${X^{-1} = \left(A^{-1}, -\Adsym{A^{-1}}{a}\right)}$ and  identity element is ${\left(I, 0\right)}$. 
For a more detailed introduction to the semi-direct product group and mathematical tools this work is built upon, we refer the reader to our previous works~\cite{Fornasier2022EquivariantBiases, Fornasier2022OvercomingCalibration}}{For a detailed introduction to \aclp{eqf} for \acl{ins}, semi-direct product groups and theoretical properties this work is built upon, we refer the reader to our previous works~\cite{Fornasier2022EquivariantBiases, Fornasier2022OvercomingCalibration, Fornasier2023EquivariantSystems}. Moreover,~\cite{Fornasier2023EquivariantSystems} discuss the advantages of semi-direct product symmetries for filter design and compares it to classical solutions such as the \acs{mekf} and the \ac{iekf}}.

\subsection{Intrinsics group $\IN$}
In this work, we recognized that elements of the camera intrinsics matrix~\cite{Hartley2004MultipleVision} form a Lie group. Thus, we introduce the intrisincs group $\IN$, as the matrix Lie group defined by
\begin{align*}
    \IN = \left\{
    \mathbf{K} = \begin{bmatrix}
         a & 0 & x \\ 0 & b & y \\ 0 & 0 & 1
    \end{bmatrix} \in \R^{3\times 3}
    \;\middle|\;
    a,b > 0, \; x,y\in \R
    \right\}.
\end{align*}
This matrix representation is associated with the standard camera intrinsics matrix, well-known in computer vision.
A typical element of $\IN$ may be written as $\mathbf{K} = (a,b,x,y)$.
Let $\mathbf{K}_1, \mathbf{K}_2 \in \IN$, then
\begin{align*}
    \mathbf{K}_1\mathbf{K}_2 &= (a_1 a_2, b_1 b_2, x_1 + a_1 x_2, y_1 + b_1 y_2), \\
    \mathbf{K}_1^{-1} &= (a_1^{-1}, b_1^{-1}, -a_1^{-1} x_1, -b_1^{-1} y_1).
\end{align*}
To the authors' understanding, exploiting the group structure of the $\IN$ group in equivariant or invariant \ac{vins} design represents a novel approach to this work.

\subsection{Useful maps}
For all $\Vector{}{v}{} = \left(x,y,z\right) \in \R^3$, define the maps
\begin{align*}
     &\pi_{Z_1}\left(\cdot\right) \AtoB{\R^{3}}{\R^{3}},
     && \pi_{Z_1}\left(\Vector{}{v}{}\right) \coloneqq \frac{\Vector{}{v}{}}{z},\\
     &\Xi\left(\cdot\right) \AtoB{\R^3}{\R^{3 \times 4}},
     &&\Xi\left(\Vector{}{v}{}\right) = \begin{bmatrix}
     x & 0 & z & 0\\
     0 & y & 0 & z\\
     0 & 0 & 0 & 0
     \end{bmatrix} \in \R^{3 \times 4}.
 \end{align*}
For all $\Vector{}{a}{}, \Vector{}{b}{}, \Vector{}{c}{} \in \R^3 \st (\Vector{}{a}{}, \Vector{}{b}{}, \Vector{}{c}{}) \in \R^9$, define the maps
 \begin{align*}
    &\Pi\left(\cdot\right) \AtoB{\se_2(3)}{\se(3)}, \quad \Pi\left(\left(\Vector{}{a}{}, \Vector{}{b}{}, \Vector{}{c}{}\right)^{\wedge}\right) = \left(\Vector{}{a}{}, \Vector{}{b}{}\right)^{\wedge} \in \se(3), \\
    &\Upsilon\left(\cdot\right) \AtoB{\se_2(3)}{\se(3)}, \quad \Upsilon\left(\left(\Vector{}{a}{}, \Vector{}{b}{}, \Vector{}{c}{}\right)^{\wedge}\right) = \left(\Vector{}{a}{}, \Vector{}{c}{}\right)^{\wedge} \in \se(3),
\end{align*}
For all $X = \left(A, a\right) \in \SE(3) \st A \in \SO(3), a \in \R^3$, define
\begin{equation*}
    \Gamma\left(\cdot\right) \AtoB{\SE(3)}{\SO(3)}, \quad \Gamma\left(X\right) = A \in \SO(3).
\end{equation*}
For all $X = \left(A, a, b\right) \in \SE_2(3) \st A \in \SO(3), a, b \in \R^3$, define
\begin{align*}
    &\chi\left(\cdot\right) \AtoB{\SE_2(3)}{\SE(3)}, \quad \chi\left(X\right) = \left(A, a\right) \in \SE(3),\\
    &\Theta\left(\cdot\right) \AtoB{\SE_2(3)}{\SE(3)}, \quad \Theta\left(X\right) = \left(A, b\right) \in \SE(3).
\end{align*}
\section{Visual Inertial Navigation System}\label{sec:vins_sys}

\subsection{System definition}
Consider a mobile platform equipped with a camera observing global visual features $\Vector{G}{p}{f}$, and an \ac{imu} providing biased acceleration and angular velocity measurements, denoted by ${\Vector{I}{w}{} = \left(\Vector{I}{\omega}{}, \Vector{I}{a}{}\right)}$.
Define ${\Pose{G}{I} = \left(\Rot{G}{I}, \Vector{G}{v}{I}, \Vector{G}{p}{I}\right)}$ to be the extended pose of the system, where $\Rot{G}{I}$ corresponds to the rigid body orientation, whereas $\Vector{G}{p}{I}$ and $\Vector{G}{v}{I}$ denote the \ac{imu} position and velocity with respect to the global frame, respectively.
Define ${\PoseP{G}{I} = \left(\Rot{G}{I}, \Vector{G}{p}{I}\right)}$\deleted{, and ${\VAtt{G}{I} = \left(\Rot{G}{I}, \Vector{G}{v}{I}\right)}$}.
Define $\Vector{I}{b}{} = \left(\Vector{I}{b}{\bm{\omega}}, \Vector{I}{b}{\bm{a}}\right)$ to be the gyroscope and accelerometer biases, respectively.
Let $g$ denote the magnitude of the acceleration due to gravity, and let $\Vector{G}{e}{3}$ denote the direction of gravity in the global frame.
Finally, define $\PoseS{I}{C}$ to be the camera extrinsic calibration, and $\mathbf{K}$ be the camera intrinsic calibration.

For the sake of readability, from now on, we suppress all the subscripts and superscripts that are not strictly required.

Define the matrices $\mathbf{W}, \mathbf{B}, \mathbf{D}, \mathbf{G}$ to be
\begin{align*}
    &\mathbf{W} = \begin{bmatrix}
    \angv^{\wedge} & \acc & \mathbf{0}_{3\times 1}\\
    \mathbf{0}_{1\times 3} & 0 & 0\\
    \mathbf{0}_{1\times 3} & 0 & 0\\
    \end{bmatrix}, \;\;
    \mathbf{B} = \begin{bmatrix}
    \bw^{\wedge} & \ba & \mathbf{0}_{3\times 1}\\
    \mathbf{0}_{1\times 3} & 0 & 0\\
    \mathbf{0}_{1\times 3} & 0 & 0\\
    \end{bmatrix},\\
    &\mathbf{D} = \begin{bmatrix}
    \mathbf{0}_{3\times 3} & \mathbf{0}_{3\times 1} & \mathbf{0}_{3\times 1}\\
    \mathbf{0}_{1\times 3} & 0 & 1\\
    \mathbf{0}_{1\times 3} & 0 & 0\\
    \end{bmatrix}, \;\;
    \mathbf{G} = \begin{bmatrix}
    \mathbf{0}_{3\times 3} & g\ethree & \mathbf{0}_{3\times 1}\\
    \mathbf{0}_{1\times 3} & 0 & 0\\
    \mathbf{0}_{1\times 3} & 0 & 0\\
    \end{bmatrix}.
\end{align*}

Finally, the visual-inertial navigation system is written
\begin{subequations}\label{eq:vins}
    \begin{align}
        &\dotPose{}{} = \Pose{}{}\left(\mathbf{W} - \mathbf{B} + \mathbf{D}\right) + \left(\mathbf{G} - \mathbf{D}\right)\Pose{}{} ,\\
        &\dotVector{}{b}{} =  \Vector{}{\tau}{} ,\\
        &\dot{\mathbf{S}} = \mathbf{S}\Vector{}{\mu}{}^{\wedge} ,\\
        &\dot{\mathbf{K}} = \mathbf{K}\Vector{}{\zeta}{}^{\wedge} ,
    \end{align}
\end{subequations}
where $\bm{\tau}, \bm{\mu}, \bm{\zeta}$ are used to model the deterministic dynamics of the bias and calibration states and are zero when these states are modeled as constants, as they are in our formulation.

Define $\xi_{I} = \left(\Pose{}{}, \Vector{}{b}{}\right) \in \torSE_2(3) \times \R^{6}$ to be the inertial navigation state.
Define $\xi_{S} = \left(\mathbf{S}, \mathbf{K}\right) \in \torSE(3) \times \torIN(3)$ to be the camera calibration state.
Then the full system state is defined as $\xi = \left(\xi_{I}, \xi_{S}\right) \in \calM \coloneqq \torSE_2(3) \times \R^{6} \times \torSE(3) \times \torIN(3)$.
Define $u = \left(\Vector{}{w}{}, \Vector{}{\tau}{}, \Vector{}{\mu}{}, \Vector{}{\zeta}{}\right) \in \mathbb{L} \subset \R^{18}$ to be the system's input.
Note that in this work, \deleted{as in other works on multi-state constraint VIO~\cite{Wu2017AnConsistency, Sun2018RobustFlight, li2013high},} visual features are not considered as part of the state since the dependency of measurement on features is removed through nullspace projection.

Without loss of generality, let us consider the case of a single feature $\Vector{}{p}{f}$.
The camera measurement is modeled as the measurement of the bearing of the feature $\Vector{}{p}{f}$ seen from the camera. 
\begin{equation}\label{eq:h}
    h\left(\xi, \Vector{}{p}{f}\right) = \mathbf{K}\pi_{Z_1}\left(\left(\mathbf{P}\mathbf{S}\right)^{-1} * \Vector{}{p}{f}\right) ,
\end{equation}
where the operation $* \AtoB{\torSE(3) \times \R^3}{\R^3}$ is defined by ${\PoseP{}{} * \Vector{}{v}{} = \Rot{}{}\Vector{}{v}{} + \Vector{}{p}{}}$  for all ${ \PoseP{}{} = \left(\Rot{}{}, \Vector{}{p}{}\right) \in \torSE(3),\; \Vector{}{v}{} \in \R^3}$.

\subsection{Symmetry of the visual-inertial navigation system}
The symmetry for the inertial navigation state $\xi_{I}$ is given by the Semi-Direct symmetry group $\grpSD \coloneqq \left(\SE_2(3) \ltimes \gothse(3)\right)$, the symmetry for the extrinsic calibration state is given by the special Euclidean group $\SE(3)$, and the symmetry for the intrinsic calibration state is given by the intrinsics group $\IN$.
The complete symmetry for the visual-inertial navigation system is thus defined to be the product group $\grpG \coloneqq \grpSD \times \SE(3) \times \IN$.

Let ${X = \left(\left(D,\delta\right), E, L\right) \in \grpG}$, with ${D = \left(A, a, b\right) \in \SE_2(3)}$ such that ${A \in \SO(3),\; a,b \in \R^{3}}$. Define the subgroups ${B = \chi\left(D\right) \in \SE(3)}$, and ${C = \Theta\left(D\right) \in \SE(3)}$. Finally, define ${E \in \SE(3)}$, and ${L \in \IN}$.
\begin{lemma}
Define ${\phi \AtoB{\grpG \times \calM}{\calM}}$ as
\begin{equation}\label{eq:phi}
    \phi\left(X, \xi\right) \coloneqq \left(\Pose{}{}D, \AdMsym{B^{-1}}\left(\Vector{}{b}{} - \delta^{\vee}\right), C^{-1}\mathbf{S}E, \mathbf{K}L\right) \in \calM .
\end{equation}
Then, $\phi$ is a transitive right group action of $\grpG$ on $\calM$.
\end{lemma}

\subsection{Lifted system}
The implementation of the equivariant filter (EqF) requires a lift ${\Lambda \AtoB{\calM \times \vecL}{\gothg}}$ to define a lifted system on the symmetry group $\grpG$ that projects down to the original system dynamics via the proposed group action $\phi$.
The transitivity of $\phi$ guarantees the existence of such a lift~\cite{Mahony2020EquivariantDesign}, and the following theorem provides an explicit form for a lift of the system studied in this paper.

\begin{theorem}
Define the map ${\Lambda \AtoB{\calM \times \vecL}{\gothg}}$ by
\begin{align*}
    \Lambda\left(\xi, u\right) &\coloneqq \left(\left(\Lambda_1\left(\xi, u\right), \Lambda_2\left(\xi, u\right)\right), \Lambda_3\left(\xi, u\right), \Lambda_4\left(\xi, u\right)\right),
\end{align*}
where ${\Lambda_1 \AtoB{\calM \times \vecL} \se_2(3)}$, ${\Lambda_2 \AtoB{\calM \times \vecL} \se(3)}$, ${\Lambda_3 \AtoB{\calM \times \vecL} \se(3)}$, and ${\Lambda_4 \AtoB{\calM \times \vecL} \mathfrak{in}}$ are given by
\begin{subequations}\label{eq:lift}
    \begin{align}
        &\Lambda_1\left(\xi, u\right) \coloneqq \left(\mathbf{W} - \mathbf{B} + \mathbf{D}\right) + \Pose{}{}^{-1}\left(\mathbf{G} - \mathbf{D}\right)\Pose{}{} ,\\
        &\Lambda_2\left(\xi, u\right) \coloneqq \left(\adMsym{\Vector{}{b}{}^{\wedge}}\left(\Pi\left(\Lambda_1\left(\xi, u\right)\right)^{\vee}\right) - \Vector{}{\tau}{}\right)^{\wedge},\\
        &\Lambda_3\left(\xi, u\right) \coloneqq \left(\AdMsym{\mathbf{S}^{-1}}\left(\Upsilon\left(\Lambda_1\left(\xi, u\right)\right)^\vee\right) + \Vector{}{\mu}{}\right)^{\wedge},\\
        &\Lambda_4\left(\xi, u\right) \coloneqq \Vector{}{\zeta}{}^{\wedge},
    \end{align}
\end{subequations}
Then ${\Lambda}$ is a lift for the system in~\equref{vins} with respect to the symmetry group $\grpG$.
\end{theorem}

The existence of the lift allows the construction of a lifted system on the symmetry group~\cite{Mahony2020EquivariantDesign}.
Let $X \in \grpG$ be the state of the lifted system, and let $\xizero = \left(\mathring{\Pose{}{}}, \mathring{\Vector{}{b}{}}, \mathring{\PoseS{}{}}, \mathring{\mathbf{K}}\right) \in \calM$ be an arbitrarily chosen element of the original state in~\equref{vins}, called the origin.
Then the lifted system is defined
\begin{equation}\label{eq:lifted_vins}
    \dot{X} = \td L_{X}\Lambda\left(\phi_{\xizero}\left(X\right), u\right).
\end{equation}
\section{Multi State Constraint Equivariant Filter}\label{sec:msceqf}
\subsection{\added{Filter state definition}}
\added{Define ${\hat{X} = \left(\left(\left(\hat{D}, \hat{\delta}\right), \hat{E}, \hat{L}\right), \hat{E}_1, \cdots, \hat{E}_k\right) \in \grpG \times \SE(3)^k}$ to be the filter's state evolving on the symmetry group. Similarly to the original formulation~\cite{Mourikis2007ANavigation} we maintain a sliding window of $k$ past $\hat{E}$ elements in the state of the filter, corresponding to the different times a camera measurement was collected.}

\subsection{Error dynamics and state transition matrix}\label{sec:errdyn}
Let ${e = \phi_{\hat{X}^{-1}}\left(\xi\right)}$ denote the equivariant error. Normal coordinates~\cite{vanGoor2022EquivariantEqF} of the state space $\calM$ in a neighborhood of the origin $\xizero$ are ${\Vector{}{\varepsilon}{} = \vartheta\left(e\right) \coloneqq \log\left(\phi_{\xizero}^{-1}\left(e\right)\right)^{\vee} \in \R^{25}}$, where ${\log \AtoB{\grpG}{\gothg}}$ is the logarithm of the symmetry group.

Recall the derivation of the linearized error dynamics in~\cite{vanGoor2022EquivariantEqF}
\begin{align*}
    &\dot{\varepsilon} \approx \mathbf{A}_{t}^{0}\varepsilon ,\\
    &\mathbf{A}_{t}^{0} = \Fr{e}{\xizero}\vartheta\left(e\right)\Fr{\xi}{\hat{\xi}}\phi_{\hat{X}^{-1}}\left(\xi\right)\Fr{E}{I}\phi_{\hat{\xi}}\left(E\right)\;\cdot\\
    & \quad\;\;\;\cdot \Fr{\xi}{\phi_{\hat{X}}\left(\xizero\right)}\Lambda\left(\xi, u\right)\Fr{e}{\xizero}
    \phi_{\hat{X}}\left(e\right)\Fr{\varepsilon}{\mathbf{0}}\vartheta^{-1}\left(\varepsilon\right) .
\end{align*}

The state matrix $\mathbf{A}_{t}^{0}$ is given by
\begin{equation}
    \mathbf{A}_{t}^{0} = 
    \begin{bmatrix}
        \prescript{}{1}{\mathbf{A}} & \prescript{}{2}{\mathbf{A}} & \Vector{}{0}{9 \times 6} & \Vector{}{0}{9 \times 4}\\
        \prescript{}{3}{\mathbf{A}} & \prescript{}{4}{\mathbf{A}} & \Vector{}{0}{6 \times 6} & \Vector{}{0}{6 \times 4}\\
        \prescript{}{5}{\mathbf{A}} & \prescript{}{6}{\mathbf{A}} & \prescript{}{7}{\mathbf{A}} & \Vector{}{0}{6 \times 4}\\
        \Vector{}{0}{4 \times 9} & \Vector{}{0}{4 \times 6} & \Vector{}{0}{4 \times 6} & \Vector{}{0}{4 \times 4}
    \end{bmatrix} \in \R^{25 \times 25},
\end{equation}
where
\begin{align*}
    & \prescript{}{1}{\mathbf{A}} =
    \begin{bNiceMatrix}[margin]
        \mathbf{\Psi} - \adMsym{\mathring{\Vector{}{b}{}}} & \Block{1-2}{\Vector{}{0}{6 \times 3}}\\
        \left(\mathring{\Rot{}{}}^T\mathring{\Vector{}{v}{}}\right)^{\wedge} - \hat{b}^{\wedge}\mathring{\Vector{}{b}{\omega}}^{\wedge} & \eye_3 & \Vector{}{0}{3 \times 3}
    \end{bNiceMatrix} \in \R^{9 \times 9},\\
    & \prescript{}{2}{\mathbf{A}} =
    \begin{bmatrix}
        \eye_3 & \Vector{}{0}{3 \times 3} \\
        \Vector{}{0}{3 \times 3} & \eye_3 \\
        \hat{b}^{\wedge} & \Vector{}{0}{3 \times 3}
    \end{bmatrix} \in \R^{9 \times 6}, \\
    & \prescript{}{3}{\mathbf{A}} = \begin{bmatrix}
        \adMsym{\mathring{\Vector{}{b}{}}}\mathbf{\Psi} - \adMsym{\left(\AdMsym{\hat{B}}\Vector{}{w}{} +  \hat{\delta}^{\vee} + \Vector{}{\theta}{}\right)}\adMsym{\mathring{\Vector{}{b}{}}} & \mathbf{0}_{6 \times 3}
    \end{bmatrix} \in \R^{6 \times 9},\\
    & \prescript{}{4}{\mathbf{A}} = \adMsym{\left( \AdMsym{\hat{B}}\Vector{}{w}{} + \hat{\delta}^{\vee} + \Vector{}{\theta}{}\right)}\in \R^{6 \times 6},\\
    & \prescript{}{5}{\mathbf{A}} = \AdMsym{\mathring{\mathbf{S}}^{-1}}
    \begin{bmatrix}
    -\Vector{}{\psi}{1}^{\wedge} & \Vector{}{0}{3 \times 3} & \Vector{}{0}{3 \times 3}\\
    -\Vector{}{\psi}{3}^{\wedge} - \mathring{\Vector{}{b}{\omega}}^{\wedge}\hat{b}^{\wedge} & \eye_3 & -\Vector{}{\psi}{2}^{\wedge}
    \end{bmatrix} \in \R^{6 \times 9},\\
    & \prescript{}{6}{\mathbf{A}} = \AdMsym{\mathring{\mathbf{S}}^{-1}}
    \begin{bmatrix}
        \eye_3 & \Vector{}{0}{3 \times 3}\\
        \hat{b}^{\wedge} & \Vector{}{0}{3 \times 3}
    \end{bmatrix} \in \R^{6 \times 6},\\
    & \prescript{}{7}{\mathbf{A}} = \adMsym{\left(\AdMsym{\mathring{\mathbf{S}}^{-1}}\Vector{}{\varrho}{}\right)} \in \R^{6 \times 6} ,
\end{align*}
with
\begin{alignat*}{2}
    &\Vector{}{\psi}{1} = \hat{A}\Vector{}{\omega}{} + \delta_{\omega}^{\vee} \in \R^{3}, & 
    & \Vector{}{\theta}{} = \left(\Vector{}{0}{3 \times 1}, g\left(\mathring{\Rot{}{}}^T\ethree\right)\right) \in \R^6,\\
    &\Vector{}{\psi}{2} = \Vector{}{\psi}{1} - \mathring{\Vector{}{b}{\omega}} \in \R^{3},&
    &\mathbf{\Psi} =
    \begin{bmatrix}
        \Vector{}{0}{3 \times 3} & \Vector{}{0}{3 \times 3} \\
        g\left(\mathring{\Rot{}{}}^T\ethree\right)^{\wedge} & \Vector{}{0}{3 \times 3}
    \end{bmatrix} \in \R^{6 \times 6} ,\\
    &\Vector{}{\psi}{3} = \hat{a} - \Vector{}{\psi}{1}^{\wedge}\hat{b} \in \R^{3}, &
    &\Vector{}{\varrho}{} = \left(\Vector{}{\psi}{2}, \Vector{}{\psi}{4}\right) \in \R^{3}. \\
    &\Vector{}{\psi}{4} = \hat{a} + \mathring{\Rot{}{}}^T\mathring{\Vector{}{v}{}} - \Vector{}{\psi}{2}^{\wedge}\hat{b} \in \R^{3}.
\end{alignat*}

The discrete-time state transition matrix is defined by ${\mathbf{\Phi} = \exp\left(\mathbf{A}_{t}^{0} \Delta T\right)}$ for time steps $\Delta T$.

\subsection{Multi state constraint}
Consider the measurement model in~\equref{h}, applying the action of the symmetry group to the state space in~\equref{phi} yields
\begin{equation}
    h\left(\phi_{X}\left(\xi\right)\right) = \mathbf{K}L \pi_{Z_1}\left(E^{-1}\left(\mathbf{P}\mathbf{S}\right)^{-1} * \Vector{}{p}{f}\right)
\end{equation}
Recall the equivariant error ${e = \phi_{\hat{X}^{-1}}\left(\xi\right) =  \vartheta^{-1}\left(\varepsilon\right)}$.
Define ${\tilde{y} = \varsigma\left(\Vector{}{p}{f}\right) - \varsigma\left(\hatVector{}{p}{f}\right)}$, where $\varsigma\left(\cdot\right)$ represents the chosen feature parametrization.
The true feature can then be written as ${\Vector{}{p}{f} = \varsigma^{-1}\left(\varsigma\left(\hatVector{}{p}{f}\right) + \tilde{y}\right)}$. 
Therefore, the measurement model in~\equref{h} can be linearized at ${\varepsilon = \mathbf{0}}$, and ${\tilde{y} = \mathbf{0}}$ as follows:
\begin{equation}\label{eq:hlin}
\begin{split}
    h\left(\xi, \Vector{}{p}{f}\right) &= h\left(\phi_{\hat{X}}\left(\vartheta^{-1}\left(\varepsilon\right)\right), \varsigma^{-1}\left(\varsigma\left(\hatVector{}{p}{f}\right) + \tilde{y}\right)\right)\\
    &= h\left(\hat{\xi}, \hatVector{}{p}{f}\right) + \mathbf{C}_t\varepsilon + \mathbf{C}^{f}_t\tilde{y} + \cdots .
\end{split}
\end{equation}

Let us derive the ${\mathbf{C}_t}$, and ${\mathbf{C}^{f}_t}$ for the \emph{anchored inverse depth} parametrization~\cite{Civera2008InverseSLAM, Mourikis2007ANavigation} of the feature. Note that 
the matrix ${\mathbf{C}^{f}_t}$ can be computed for any desired parametrization.

Let ${^A\mathbf{P} ^A\mathbf{S}}$ be the pose of the anchor, defined as the pose of the camera where the feature ${\Vector{}{p}{f}}$ has been first seen. Define the feature in the anchor frame as ${\Vector{}{a}{f} = \left(^A\mathbf{P} ^A\mathbf{S}\right)^{-1} * \Vector{}{p}{f}}$, with ${\Vector{}{a}{f} = \left(a_{f_x}, a_{f_y}, a_{f_z}\right) \in \R^3}$. The anchored inverse depth parametrization is written
\begin{align}\label{eq:aid_param}
    &\Vector{}{z}{} = \varsigma\left(\Vector{}{p}{f}\right) = \left(\Vector{}{z}{1}, z_2\right) = \left(\left(\frac{a_{f_x}}{a_{f_z}}, \frac{a_{f_y}}{a_{f_z}}\right), \frac{1}{a_{f_z}}\right) ,\\
    &\Vector{}{p}{f} = \varsigma^{-1}\left(\Vector{}{z}{}\right) = \left(^A\mathbf{P} ^A\mathbf{S}\right) * \begin{bmatrix}
        \frac{\Vector{}{z}{1}}{z_2}\\
        \frac{1}{z_2}
    \end{bmatrix} .
\end{align}
Then the matrix ${\mathbf{C}^{f}_t}$ is written
\begin{equation}\label{eq:Cf}
    \begin{split}
        \mathbf{C}^{f}_t\tilde{y} &= 
        \mathring{\mathbf{K}}
        \hat{L}
        d_{\pi_{Z_1}}
        \Gamma\left(\left(\hat{\mathbf{P}}\hat{\mathbf{S}}\right)^{-1} {^A\hat{\mathbf{P}}} {^A\hat{\mathbf{S}}}\right)\frac{1}{\hat{z}_2}
        \begin{bmatrix}
            \eye_2 & -\frac{\hatVector{}{z}{1}}{\hat{z}_2}\\
            \Vector{}{0}{1 \times 2} & -\frac{1}{\hat{z}_2}
        \end{bmatrix}
        \tilde{y} \\
        &= 
        \mathring{\mathbf{K}}
        \hat{L}
        d_{\pi_{Z_1}}
        \Gamma\left({\hat{E}^{-1}} {^A\hat{E}}\right)\frac{1}{\hat{z}_2}
        \begin{bmatrix}
            \eye_2 & -\frac{\hatVector{}{z}{1}}{\hat{z}_2}\\
            \Vector{}{0}{1 \times 2} & -\frac{1}{\hat{z}_2}
        \end{bmatrix}
        \tilde{y},
    \end{split}
\end{equation}
where we have used ${\hat{\xi} \coloneqq \phi_{\hat{X}}\left(\xizero\right)}$ to map between the estimated state in the homogeneous space $\hat{\xi}$, and the estimated state in the symmetry group $\hat{X}$. Therefore
\begin{equation*}
    \left(\hat{\mathbf{P}}\hat{\mathbf{S}}\right)^{-1} {^A\hat{\mathbf{P}}} {^A\hat{\mathbf{S}}} = \hat{E}^{-1}\mathring{\PoseS{}{}}^{-1}\hat{C}\hat{C}^{-1}\mathring{\PoseP{}{}}^{-1}\mathring{\PoseP{}{}}^A\hat{C}^A\hat{C}^{-1}\mathring{\PoseS{}{}}^A\hat{E} = \hat{E}^{-1}{^A\hat{E}} .
\end{equation*}

According to~\cite{vanGoor2022EquivariantEqF}, the ${\mathbf{C}_t}$ matrix is defined by
\begin{equation}\label{eq:Ct}
    \begin{split}
        \mathbf{C}_t\varepsilon &= \Fr{\xi}{\hat{\xi}}h\left(\xi\right)\Fr{e}{\mathring{\xi}}\phi_{\hat{X}}\Fr{\varepsilon}{\mathbf{0}}\vartheta^{-1}\left(\varepsilon\right)\left[\varepsilon\right]\\
        &= \mathring{\mathbf{K}}
        \hat{L}
        d_{\pi_{Z_1}}\Gamma\left(\hat{E}^{-1}\right)
        \begin{bmatrix}
        \left({^AE}\hatVector{}{a}{f}\right)^{\wedge} & -\eye_3
        \end{bmatrix}
        \varepsilon_{E}\; -\\
        &- \mathring{\mathbf{K}}
        \hat{L}
        d_{\pi_{Z_1}}\Gamma\left(\hat{E}^{-1}\right)
        \begin{bmatrix}
        \left({^AE}\hatVector{}{a}{f}\right)^{\wedge} & -\eye_3
        \end{bmatrix}
        \varepsilon_{^AE}\; +\\
        &+
        \mathring{\mathbf{K}}\Xi\left(\hat{L}\pi_{Z_1}\left(\hat{E}^{-1}{^AE}\hatVector{}{a}{f}\right)\right)
        \varepsilon_L ,
    \end{split}
\end{equation}
where $\varepsilon_{E}$, and $\varepsilon_{^AE}$ represent respectively the error in normal coordinates for the element $E$ of the symmetry group corresponding to the most recent pose and to the anchor pose, whereas $\varepsilon_{L}$ represent the error in normal coordinates that is related to the camera intrinsics.

\deleted{From the preceding expressions, the structure of the \ac{msceqf} becomes clear. Similarly to the original formulation~\cite{Mourikis2007ANavigation} we maintain a sliding window of $E$ elements in the state of the filter, corresponding to the different times a camera measurement is collected.}

To compute the matrix ${\mathbf{C}_t}$ in~\equref{Ct}, an estimate of the feature position in the anchor frame is required. To this end, when a feature has been seen from multiple views a linear-nonlinear least square problem can be solved~\cite{Mourikis2007ANavigation, Geneva2020OpenVINS:Estimation}.

Finally, to remove the dependency of the features, and hence perform a filter update, we employ nullspace marginalization of the matrix ${\mathbf{C}^{f}_t}$ in \equref{hlin}, according to the original formulation~\cite{Mourikis2007ANavigation}.
\begin{figure*}
\centering
\includegraphics[width=\linewidth]{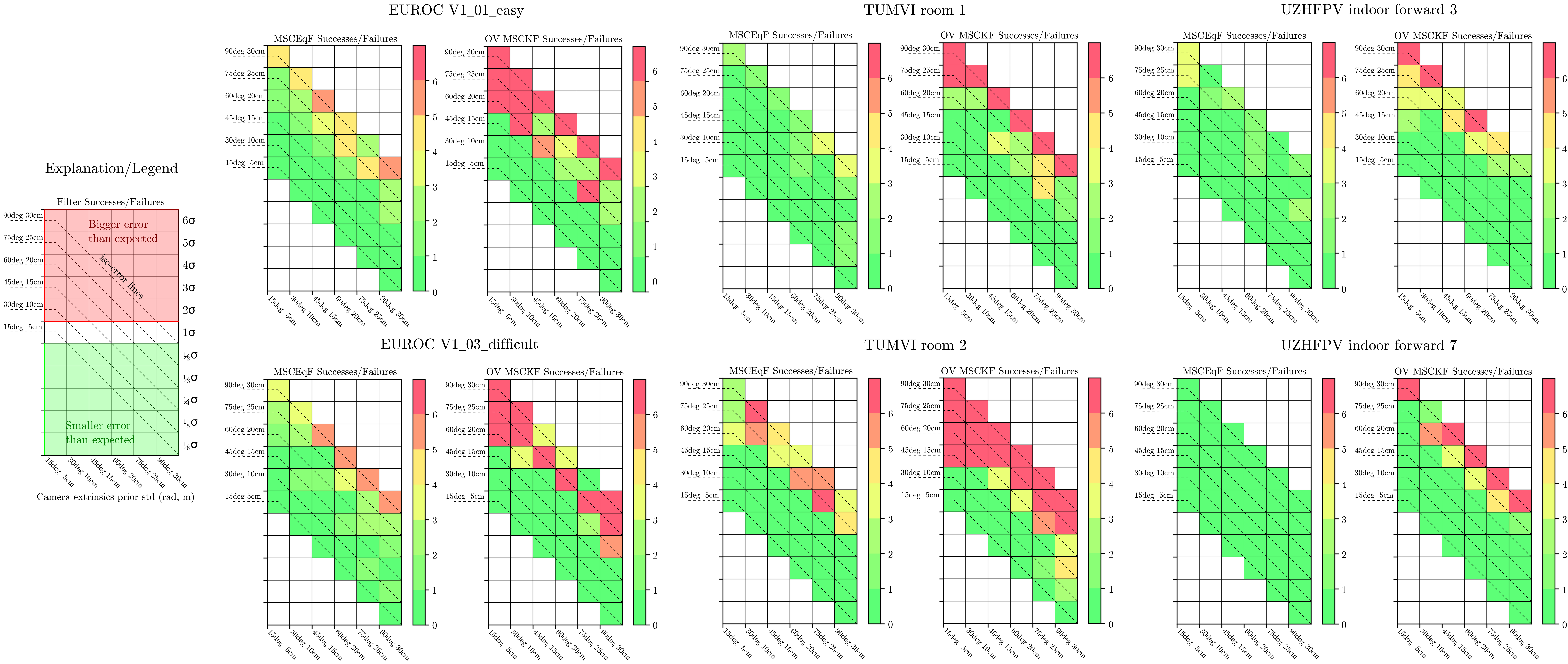}
\caption{Results of the experiment \deleted{aiming at} evaluating the robustness of the proposed \ac{msceqf} and \changed{OpenVINS}{OpenVINS's \ac{msckf}}. In these grid plots, the x-axis is the prior standard deviation the estimators are set with. The y-axis is how many $\sigma$-levels that error corresponds to. Labeled diagonal dashed lines represent iso-error lines (lines along with the error is constant). The bottom part of each grid represents expected errors, thus errors falling within $\sfrac{1}{6} \sigma$-$\sfrac{1}{2} \sigma$, whereas the top part of each grid represents unexpected errors, thus errors falling within $2 \sigma$-$6 \sigma$. According to the colorbar, the color of each cell shows the number of failures.}
\label{fig:grid}
\end{figure*}

\section{Experiments}\label{sec:exp}
In this letter, we perform a series of experiments to evaluate the accuracy, consistency, and, more importantly, robustness of the proposed \ac{msceqf}. We perform many experiments on real-world data to evaluate robustness to expected and unexpected errors in the camera extrinsic calibration. In all these experiments, we limit our comparison to filter-based \ac{msckf} algorithms for \ac{vio}, and in particular, to the best available one we believe represents the state-of-the-art, that is OpenVINS~\cite{Geneva2020OpenVINS:Estimation}. For a fair comparison, we turned off OpenVINS's persistent features (\ac{slam} features), and only compare against its pure \changed{multi state constraint filter}{\ac{msckf}} part. Furthermore, in all the experiments, OpenVINS\added{'s {\ac{msckf}}} parameters were specifically tuned, for each dataset, according to the authors' suggested parameters. In contrast, the proposed \ac{msceqf} shares the same tuning parameters across all the experiments and datasets.

\subsection{Robustness}
Robustness is an important property of a modern filter-based \acl{vio} algorithm. It is the ability to function with significant yet known errors, as well as the ability to deal with unknown unknowns. In simpler terms, it refers to how well an algorithm performs under non-ideal conditions, such as imperfect tuning parameters, poor calibration, or unexpected changes in the sensor's extrinsic parameters during field operations.

To assess the robustness of the proposed \ac{msceqf} and \changed{OpenVINS}{the \ac{msckf}}, we ran a series of experiments using widely-known dataset for evaluating \ac{vio} algorithms. Specifically, the Euroc dataset~\cite{Burri2016TheDatasets}, the TUM-VI dataset~\cite{Schubert2018TheOdometry}, and the UZH-FPV dataset~\cite{Delmerico19icra}. For each dataset, we selected two sequences and ran each estimator $6\times6\times6 = 216$ times (for a total number of runs of $2592$). In these experiments, we intentionally initialized the filters with incorrect camera extrinsic parameters, introducing errors in six steps ranging from ($15$\si[per-mode = symbol]{\degree}, $0.05$\si[per-mode = symbol]{\meter}), to ($90$\si[per-mode = symbol]{\degree}, $0.3$\si[per-mode = symbol]{\meter}). For each error step, we ran the estimators with six different priors (initial covariance) accounting for initial calibration errors in the range of the six error steps. For each pair (prior, error) we run each estimator six times. Finally, for each individual run, we classified an estimator as converged or diverged based on a position error threshold.

Based on the results of the experiment in \figref{grid}, we derive the following noteworthy observations. In absolute terms, there seems to be an upper limit of absolute error that, no matter the prior, makes the estimators diverge. Although this limit highly depends on the dataset, for each of the tested sequences, the proposed \ac{msceqf} possesses a higher error limit, and hence improved robustness to known absolute error.
In relative terms, the proposed \ac{msceqf} seems to deal better with unknown errors since the line at which the estimator fails is straight and does not bend towards the left side as it appears to happen for \changed{OpenVINS}{the \ac{msckf}}. Encouraged by these results, we ran an additional experiment on the \emph{V1\_01\_easy} sequence of the Euroc dataset, introducing new, smaller priors and errors to effectively evaluate whether the estimators are able to manage errors that are smaller in absolute terms but outside the prior covariance. \figref{grid_extension} clearly shows that the \ac{msceqf} is indeed a more robust filter, able to deal with unexpected errors. \added{Finally, \figref{convergence} shows the convergence of the camera extrinsic parameters for both filters evaluated on the Euroc \emph{V1\_01\_easy} sequence, with an initial error of ($30$\si[per-mode = symbol]{\degree}, $0.1$\si[per-mode = symbol]{\meter}) and an initial covariance to match the error. The error plots clearly show that the proposed \ac{msceqf} not only is a more robust filter, but it also converges faster.}

Quantifying robustness in robotics, however, remains an ongoing challenge. 
In the presented evaluation, we have chosen the camera extrinsic calibration as a state subjected to error. Even though static and dynamic initialization approach exists~\cite{Dong-Si2012EstimatorCalibration, Campos2019FastSLAM} for such a problem, in our formulation, extrinsic parameters are treated as regular state variables, and our proposed algorithm showcases inherent robustness by successfully attaining reliable estimation, for both expected and unexpected errors, eliminating the need of any auxiliary module.
This characteristic sets our algorithm apart from conventional \ac{vio} algorithms, emphasizing its superior robustness.

\begin{figure}[!t]
    \centering
    \includegraphics[width=\linewidth]{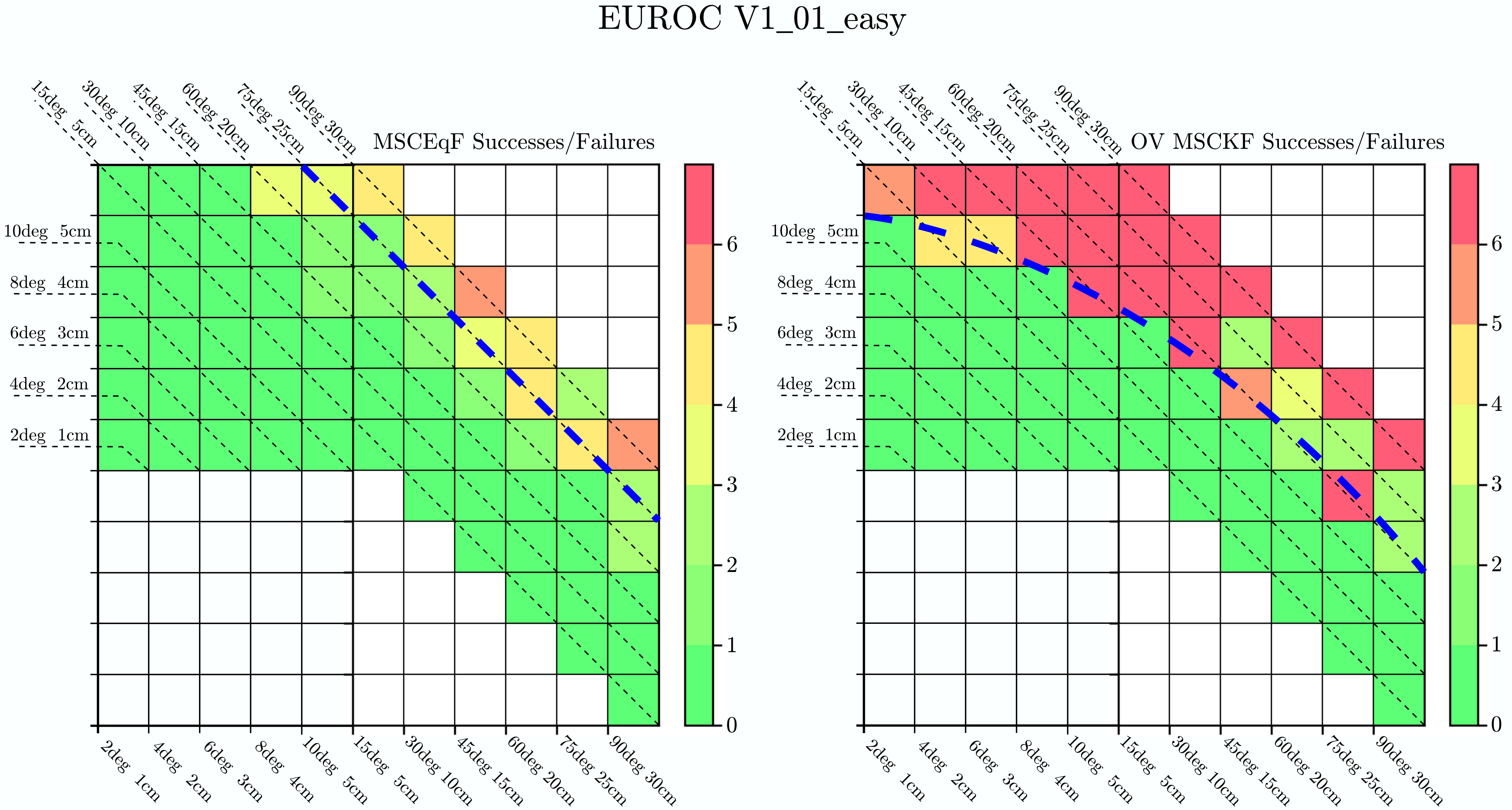}
    \caption{Grid plot showing the robustness of the proposed \ac{msceqf} compared to \changed{OpenVINS}{OpenVINS's \ac{msckf}} for unexpected errors, thus the ability to deal with \emph{you don't know what you don't know}. The x-axis is the prior standard deviation the estimators are set with. The y-axis is how many $\sigma$-levels that error corresponds to. Diagonal dashed lines represent iso-error lines. The blue bold dashed line is the limit at which each estimator fails. According to the colorbar, the color of each cell represents the number of failures.}
    \label{fig:grid_extension}
\end{figure}

\subsection{Accuracy}
Our next experiment focuses on the classical and widely-used metric for evaluating the performance of \acl{vio} algorithms~\cite{Delmerico2018ARobots}, namely the \ac{rmse} of the \ac{ate}. For this experiment, we ran the proposed \ac{msceqf} and OpenVINS\added{'s \ac{msckf}} on all Euroc sequences~\cite{Burri2016TheDatasets}\deleted{, except for the \emph{V2\_03\_difficult} sequence, due to unreliable results for both algorithms}. The results presented in \tabref{rmse} demonstrate that the proposed \ac{msceqf} achieves state-of-the-art accuracy comparable to \changed{OpenVINS}{the \ac{msckf}}. It should be noted that in our evaluation, we aligned each estimate with the groundtruth using the initial state rather than finding the optimal alignment that minimizes the error throughout the entire trajectory.

\begin{table*}
\renewcommand{\arraystretch}{1.15}
\setlength\tabcolsep{4.0pt}
\centering
\scriptsize
\caption{Attitude (A), and position (P) Absolute Trajectory Error (ATE) RMSE on Euroc dataset}
\vspace{-5pt}
\begin{tabular}{c|cc|cc||c|cc|cc||c|cc|cc}
\specialrule{.1em}{.1em}{.1em}
\textsc{Sequence} & \multicolumn{2}{c|}{\textsc{MSCEqF}} & \multicolumn{2}{c||}{\textsc{OV MSCKF}} & \textsc{Sequence} & \multicolumn{2}{c|}{\textsc{MSCEqF}} & \multicolumn{2}{c||}{\textsc{OV MSCKF}} & \textsc{Sequence} & \multicolumn{2}{c|}{\textsc{MSCEqF}} & \multicolumn{2}{c}{\textsc{OV MSCKF}} \\ 
\hline
& A [\si[per-mode = symbol]{\radian}] & P [\si[per-mode = symbol]{\meter}] & A [\si[per-mode = symbol]{\radian}] & P [\si[per-mode = symbol]{\meter}] & & A [\si[per-mode = symbol]{\radian}] & P [\si[per-mode = symbol]{\meter}] & A [\si[per-mode = symbol]{\radian}] & P [\si[per-mode = symbol]{\meter}] & & A [\si[per-mode = symbol]{\radian}] & P [\si[per-mode = symbol]{\meter}] & A [\si[per-mode = symbol]{\radian}] & P [\si[per-mode = symbol]{\meter}] \\ 
\hline
V1\_01\_easy & $0.07$ & \color{blue}{$\mathbf{0.24}$} & \color{blue}{$\mathbf{0.05}$} & $0.36$ & V2\_02\_medium & $0.08$ & $0.55$ & \color{blue}{$\mathbf{0.03}$} & \color{blue}{$\mathbf{0.17}$} & MH\_03\_medium & $0.02$ & \color{blue}{$\mathbf{0.34}$} & \color{blue}{$\mathbf{0.01}$} & $0.41$\\
V1\_02\_medium & $0.03$ & \color{blue}{$\mathbf{0.20}$} & \color{blue}{$\mathbf{0.02}$} & $0.22$ & V2\_03\_difficult\footnotemark{} & \color{blue}{$\mathbf{0.03}$} & $0.39$ & \color{blue}{$\mathbf{0.03}$} & \color{blue}{$\mathbf{0.28}$} & MH\_04\_difficult & \color{blue}{$\mathbf{0.03}$} & \color{blue}{$\mathbf{0.53}$} & $0.04$ & $0.61$\\
V1\_03\_difficult & $0.05$ & $0.30$ & \color{blue}{$\mathbf{0.02}$} & \color{blue}{$\mathbf{0.18}$} & MH\_01\_easy & \color{blue}{$\mathbf{0.05}$} & \color{blue}{$\mathbf{0.29}$} & \color{blue}{$\mathbf{0.05}$} & $0.42$ & MH\_05\_difficult & $0.03$ & \color{blue}{$\mathbf{0.70}$} & \color{blue}{$\mathbf{0.02}$} & $0.78$\\
V2\_01\_easy & \color{blue}{$\mathbf{0.02}$} & \color{blue}{$\mathbf{0.13}$} & $0.05$ & $0.18$ & MH\_02\_easy & \color{blue}{$\mathbf{0.01}$} & \color{blue}{$\mathbf{0.38}$} & $0.03$ & $0.54$ & & & & & \\
\specialrule{.1em}{.1em}{.1em} 
\end{tabular}
\label{tab:rmse}
\end{table*}

\subsection{Consistency}
An estimator is said to be consistent if the estimated covariance of the error reflects its real distribution; in other words, an estimator is consistent if the error is unbiased and within the sigma bounds of the estimated covariance. 
\added{Consistency of the proposed \ac{msceqf} is proven by compatibility of the group action $\phi$ in \equref{phi}, and invariance of the lift $\Lambda$ in \equref{lift}, to reference frame transormations~\cite{vanGoor2023EqVIO:Odometry, Wu2017AnConsistency}. This ensures that the filter does not gain spurious information along the unobservable directions.}


\begin{theorem}
Define $H{ \coloneqq (R_{H}, 0, p_{H}) \in \SE_2(3)}$, where ${R_{H} \in \SE_{\ethree}(3)}$ represent a anti-clockwise rotation about the vertical axis $\ethree$, and $p_{H}$ represent the a translation. Define the right group action ${\alpha \AtoB{\SE_2(3) \times \calM}{\calM}}$ such that ${\alpha(H, \xi) \coloneqq (H^{-1}\Pose{}{}, \Vector{}{b}{}, \mathbf{S}, \mathbf{K})}$ represents a change of reference, from \frameofref{G} to \frameofref{H} that leaves the direction of gravity unchanged.

Then the action of the symmetry group on the state space $\phi$ and the lift $\Lambda$ are respectively compatible and invariant with respect to change of reference, that is
\begin{align*}
    &\alpha(H,\phi(X,\xi)) = \phi(X, \alpha(H, \xi)),\\
    &\Lambda(\alpha(H,\xi), u) = \Lambda(\xi, u) .
\end{align*}
\end{theorem}
\begin{proof}
\begin{align*}
    \phi(X, \alpha(H, \xi)) &= ((H^{-1}\Pose{}{})D, \AdMsym{B^{-1}}\left(\Vector{}{b}{} - \delta^{\vee}\right), C^{-1}\mathbf{S}E, \mathbf{K}L)\\
    &= (H^{-1}\Pose{}{}D, \AdMsym{B^{-1}}\left(\Vector{}{b}{} - \delta^{\vee}\right), C^{-1}\mathbf{S}E, \mathbf{K}L)\\
    &= \alpha(H,\phi(X,\xi)) ,
\end{align*}
as required.

To prove the invariance of $\Lambda$ to the action $\alpha$, it is sufficient to show that ${\Lambda_1(\alpha(H,\xi), u) = \Lambda_1(\xi, u)}$.
\begin{align*}
\Lambda_1(\alpha(H,\xi), u) &= (\mathbf{W} - \mathbf{B} + \mathbf{D}) + (\Pose{}{}^{-1}H)(\mathbf{G} - \mathbf{D})(H^{-1}\Pose{}{}) \\
&= (\mathbf{W} - \mathbf{B} + \mathbf{D}) + \Pose{}{}^{-1}(H(\mathbf{G} - \mathbf{D})H^{-1})\Pose{}{} \\
&= (\mathbf{W} - \mathbf{B} + \mathbf{D}) + \Pose{}{}^{-1}(H(\mathbf{G} - \mathbf{D})H^{-1})\Pose{}{} \\
&= (\mathbf{W} - \mathbf{B} + \mathbf{D}) + \Pose{}{}^{-1}(\mathbf{G} - \mathbf{D})\Pose{}{} \\
&= \Lambda_1(\xi, u) ,
\end{align*}
where we have used the fact that ${H(\mathbf{G-D})H^{-1} = \mathbf{G-D}}$. Specifically
\begin{align*}
    H(\mathbf{G-D})H^{-1} &= \begin{bmatrix}
    \mathbf{0}_{3\times 3} & R_{H}g\ethree & \mathbf{0}_{3\times 1}\\
    \mathbf{0}_{1\times 3} & 0 & -1\\
    \mathbf{0}_{1\times 3} & 0 & 0\\
    \end{bmatrix} \\
    &= \begin{bmatrix}
    \mathbf{0}_{3\times 3} & g\ethree & \mathbf{0}_{3\times 1}\\
    \mathbf{0}_{1\times 3} & 0 & -1\\
    \mathbf{0}_{1\times 3} & 0 & 0\\
    \end{bmatrix} \\
    &= \mathbf{G-D} .
\end{align*}
It is straightforward to see that ${R_{H}g\ethree = g\ethree}$ since $R_{H}$ is a rotation about the $\ethree$ axis. This completes the proof.
\end{proof}


In this final experiment, we employed the pose (orientation and position) \ac{anees} as a metric to analyze the consistency of the proposed \ac{msceqf}. In particular, we used the VINSEval framework~\cite{Fornasier2021} to generate a photorealistic synthetic dataset of 25 runs of the same trajectory, with the same noise statistics but different noise realizations.

The \ac{anees} for the \ac{msceqf} was computed according to the following formula
\begin{equation*}
    \text{ANEES} = \frac{1}{Mn}\sum_{i=1}^{M}\bm{\varepsilon}_i^T\bm{\Sigma}_i^{-1}\bm{\varepsilon}_i ,
\end{equation*}
where $M$ is the number of runs, $n = dim\left(\bm{\varepsilon}\right)$ is the dimension of the error $\bm{\varepsilon}$, and $\bm{\Sigma}$ is the covariance of the error. The error $\bm{\varepsilon} = \log_{\SE(3)}\left(\mathring{\PoseP{}{}}^{-1}\PoseP{}{}\hatPoseP{}{}^{-1}\mathring{\PoseP{}{}}\right)^{\vee}$ is the pose components of the equivariant error defined in \secref{errdyn}.

The resulting \ac{anees} shown in \figref{anees} fluctuates around a computed average of $1.0$ and is not increasing or decreasing over time. This is a very similar average than \ac{fej} estiamtors~\cite{Geneva2020OpenVINS:Estimation, Chen2022FEJ2:Design}, but without requiring artificial modification of the linearization points to achieve consistency.

\begin{figure}[!t]
    \centering
    \includegraphics[width=\linewidth]{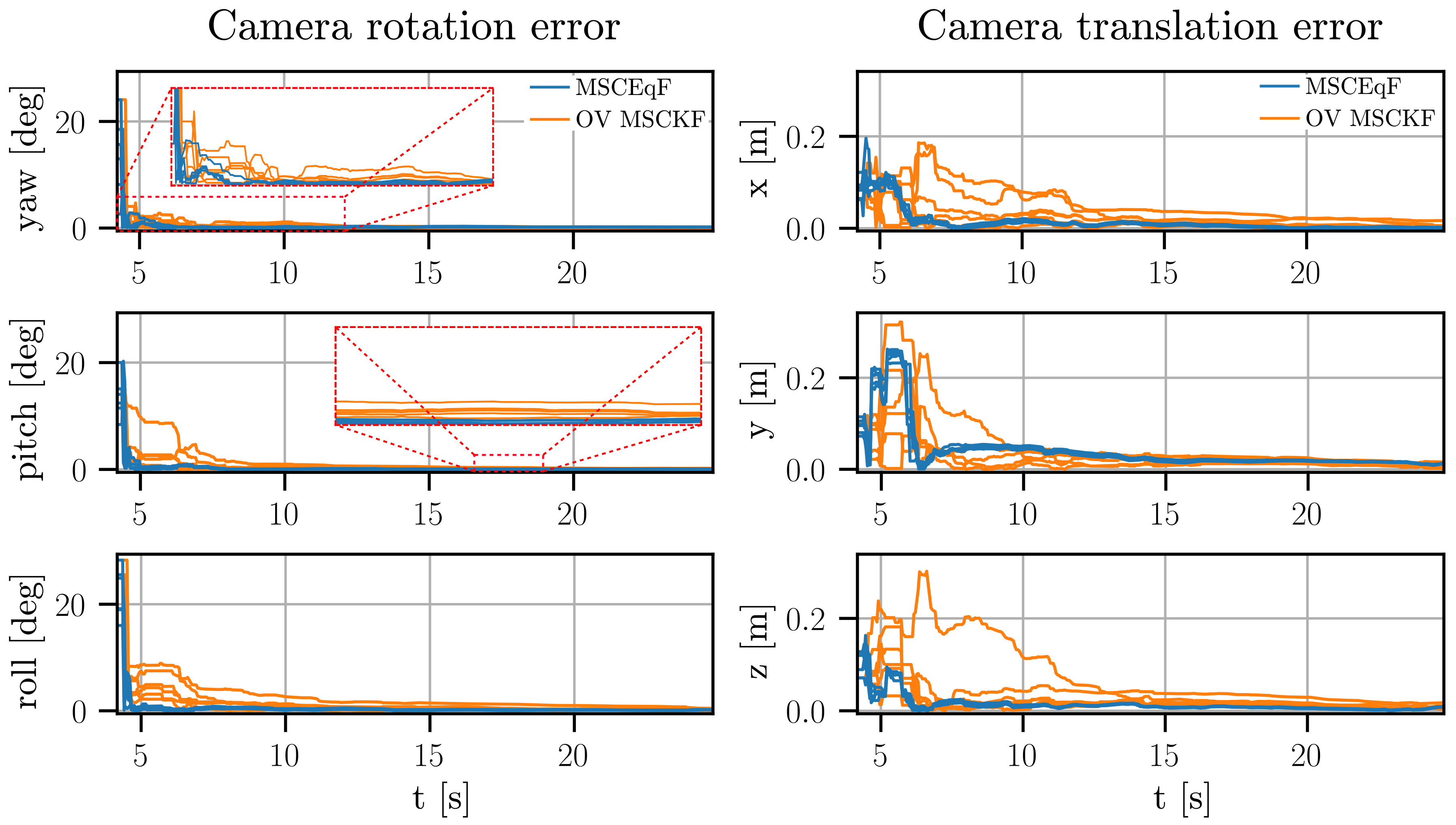}
    \caption{\added{Absolute errors of camera extrinsic parameters for the proposed \ac{msceqf}, and OpenVINS's MSCKF. The plots show the convergence performance of the filters evaluated on the Euroc \emph{V1\_01\_easy} sequence, for $6$ runs, with an initial error of ($30$\si[per-mode = symbol]{\degree}, $0.1$\si[per-mode = symbol]{\meter}).}}
    \label{fig:convergence}
\end{figure}

\begin{figure}[!t]
\centering
\includegraphics[width=\linewidth]{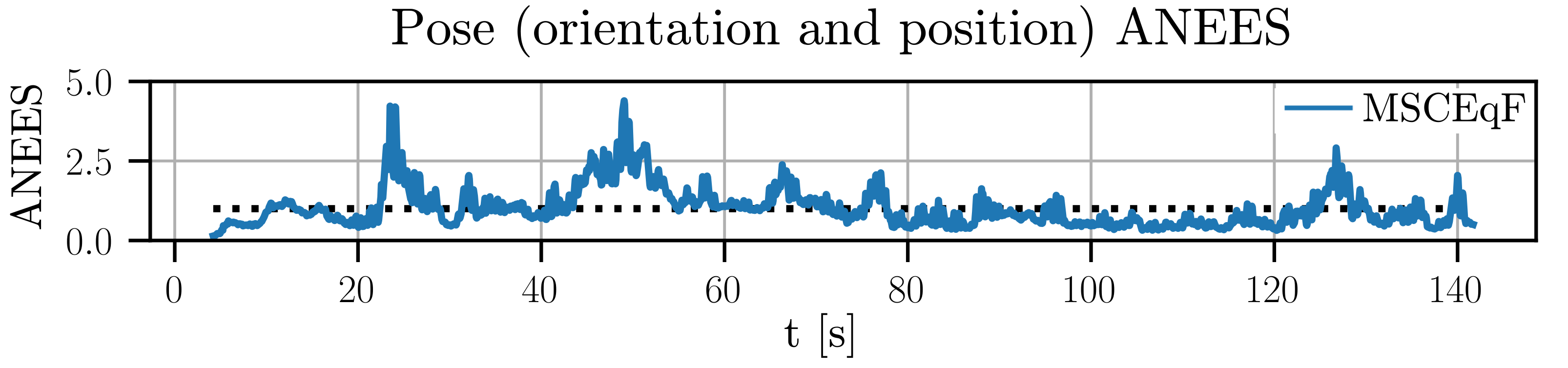}
\caption{Pose (orientation and position) \acs{anees} of the proposed \ac{msceqf} \deleted{executed} for 25 runs on a custom \deleted{photorealistic} dataset generated with the VINSEval framework. \deleted{The plot shows that the \acs{anees} is fluctuating around an average of $1.0$ (dashed line) and does not increase or decrease over time.}}
\label{fig:anees}
\end{figure}

\footnotetext{Due to non-deterministic results with varying in accuracy, we reported the best result out of $5$ runs}
\section{Conclusion}\label{sec:conc}
This letter presented the \emph{\acf{msceqf}}. A novel \acl{eqf} formulation for the \ac{vio} problem, capable of camera intrinsic and extrinsic self-calibration.
With our approach, we address the need for an \ac{vio} algorithm that achieves state-of-the-art accuracy and consistency while minimizing the need for sophisticated tuning and remaining robust against expected \emph{and unexpected} errors.
Through the presented experiments, we have demonstrated that the proposed \ac{msceqf} successfully tackles these requirements. It exhibits robustness against both high absolute errors and unexpected errors that exceed the prior covariance. Furthermore, the \ac{msceqf} has been proven to be a naturally consistent estimator, achieving accuracy comparable to a state-of-the-art \ac{msckf} algorithm but without the need for additional health-check nor consistency enforcing modules and heuristics. \added{Future work includes the extension of the proposed \ac{msceqf} with a polar symmetry for explicit \ac{slam} features~\cite{vanGoor2023EqVIO:Odometry}}

\bibliographystyle{IEEEtran}
\bibliography{bibliography/extra, bibliography/IEEEabrv, bibliography/main, bibliography/b042115, bibliography/vio}


\end{document}